\documentclass[journal]{IEEEtran}
\usepackage{blindtext}
\usepackage{graphicx}
\usepackage{amsthm}
\usepackage{amsmath}
\usepackage{algorithm}
\usepackage[noend]{algpseudocode}
\usepackage{verbatim}
\usepackage{multicol}
\usepackage{tikz}
\usepackage{subcaption}
\usepackage{caption}
\usepackage{placeins}
\usepackage{hyperref}
\usetikzlibrary{arrows}
\theoremstyle{definition}
\usepackage[font={small,it}]{caption}
\newtheorem{theorem}{Theorem}[section]
\newtheorem{corollary}{Corollary}[theorem]
\newtheorem{lemma}[theorem]{Lemma}
\newtheorem{proposition}[theorem]{Proposition}

\newtheorem{definition}{Definition}[section]
\newtheorem{hypothesis}{Hypothesis}[section]
\usepackage{tikz}
\newtheorem{example}[theorem]{Example}

\usepackage{cite}

\ifCLASSINFOpdf
\else
\fi

\begin{document}
\title{A natural approach to studying schema processing \\
\large{ Submitted to the IEEE Transactions on Evolutionary Computation on 11/05/2017 }}
%
%
%

\author{Jack McKay Fletcher and
        Thomas Wennekers
\thanks{Jack McKay Fletcher and  Thomas Wennekers are with the Centre for Robotics and Neural Systems (CRNS),
University of Plymouth,
Plymouth, Devon, United Kingdom
(contact jack.mckayfletcher@plymouth.ac.uk).}}
\maketitle
\begin{abstract}
The Building Block Hypothesis (BBH) states that adaptive systems combine good partial solutions (so-called building blocks) to find increasingly better solutions. It is thought that Genetic Algorithms (GAs) implement the BBH. However, for GAs building blocks are semi-theoretical objects in that they are thought only to be implicitly exploited via the selection and crossover operations of a GA. In the current work, we discover a mathematical method to identify the complete set of schemata present in a given population of a GA; as such a natural way to study schema processing (and thus the BBH) is revealed. We demonstrate how this approach can be used both theoretically and experimentally. Theoretically, we show that the search space for good schemata is a complete lattice and that each generation samples a complete sub-lattice of this search space. In addition, we show that combining schemata can only explore a subset of the search space. Experimentally, we compare how well different crossover methods combine building blocks. We find that for most crossover methods approximately 25-35\% of building blocks in a generation result from the combination of the previous generation's building blocks. We also find that an increase in the combination of building blocks does not lead to an increase in the efficiency of a GA. To complement this article, we introduce an open source Python package called {\it schematax}, which allows one to calculate the schemata present in a population using the methods described in this article.

\end{abstract}


%
\IEEEpeerreviewmaketitle

\section{Introduction}

Genetic Algorithms (GAs) are a hugely popular method for optimization and have found successes on many problems \cite{1134124}. Sadly, unlike other optimization techniques such as gradient decent \cite{kiwiel2001convergence}, simulated annealing \cite{granville1994simulated,belisle1992convergence,hwang1988simulated}, Ant Colony Optimization \cite{stutzle2002short,dorigo2006ant} or Particle Swarm Optimisation \cite{clerc2002particle,trelea2003particle}, GAs lack a rigorous explanation of exactly why and on what functions they perform well. There is, however, a chief approach to studying the power GAs, which is by considering the schemata GAs are processing.
\\
\\
Schemata are simple mathematical objects which describe points and hyper planes in the space of all possible words over an alphabet of the same length \cite{john1992holland}. Specifically, a schema is a word made with an additional symbol $*$ called the wild card symbol, which stands for `dont't care'. For example, the schema over the binary alphabet $1*1*$ represents the set of binary strings which have a $1$ in positions one and three and a $1$ or $0$ in positions two and four. In this way, the wild card symbol is similar to a blank tile in the popular board game Scrabble. Schemata have properties. For a schema $s$, the order of $s$, denoted $o(s)$ is the number of non-wild card symbols (that is symbols which are not `$*$') in $s$. For example, the order of the above schema is $2$. The defining length of $s$, denoted $d(s)$ is the distance between the first and last non wild card symbol. The above schema has a defining length of 2. A word is said to be an instance of $s$ if it matches $s$. For example, the word $1111$ is an instance of the above schema. In the context of a GA, the fitness of $s$ in a population is the average fitness of all of it's instances.
\\
\\
Holland, in his partly philosophical work on adaptation, argues that any adaptive process test subsets of the search space through schema \cite{john1992holland}. As each individual in the search space belongs to several schemata at once, by evaluating one individual many schemata are implicitly sampled. The idea being, if an individual is fit it suggests the schemata of which it is an instance are also fit. Thus, by testing a few individuals many schemata are sampled. This property is called {\it implicit} or {\it intrinsic} parallelism. Holland claims that natural evolution exploits this property.  From this concept, Holland created a statement about the schema processing performed by GAs: 
\begin{definition}
A building block is a low order, low defining length and above average fitness schema.
\end{definition}
\begin{hypothesis}
{\bf  The BBH} (The building block hypothesis): competent GAs find increasingly better solutions by combining building blocks.  
\end{hypothesis}
Holland's idea of building blocks is threefold. Firstly, as building blocks have above average fitness, they have a high probability of surviving and generating offspring. Secondly, as they have a low order, they have a low probability of being disrupted by mutations. Thirdly, as they have a low defining length they have a low probability of disruption because of crossover. All of these properties point towards the BBH, that is: building blocks surviving and being combined in subsequent generations. Note that the BBH is a statement about adaptive systems in general, but in this case it is applied specifically to GAs.   
\\
\\
The BBH was the putative explanation for the power of GAs for a long time. Schema theorems, which provide a lower bound on the expected number of instances of a schema in one generation occurring in the subsequent generation \cite{bridges1987analysis,poli2000exact,goldberg2001practical} seemed to add credence to the BBH as they show that building blocks have a high chance of surviving. However, in later times the BBH came under many philosophical and theoretical criticisms.  In his paper titled 'The building block fallacy' \cite{Thornton97thebuilding}, Thorton questions the reasoning leading to the BBH and proposes contradictions between the schema theorem and the BBH. Others, such as Vose calls the earlier theory of GA "myths and folklore" and argues there is a lack of a standard GA theory \cite{vose1999simple}.  In Holland's framing, the schemata being manipulated by a GA are semi-theoretical objects in that they are not directly manipulated by the genetic algorithm. Rather, it is proposed that the distribution of offspring should change {\it as if} the schemata of the parents had been sampled and combined. 
\\
\\
There is one article which studies schema processing in a non-theoretical manner. Namely, Mitchell et al.'s \cite{mitchell1992royal} insightful work on building blocks. In this article, building blocks are indirectly studied through ``Royal Road Functions". ``Royal Road Functions'' are fitness functions which have building blocks explicitly built into them. These functions reward individuals for finding good partial solutions, in a way setting up `stepping stones' along the way to the optimal solution. Assuming the BBH, one would expect to find an optimal solution  very quickly as `building blocks' are written directly into the fitness function.  However, the authors find that when the fitness function {\it does not} have stepping stones it performs better, that is a GA finds the optimal solution in a fewer number of steps. It is suggested that the reward for partial solutions in Royal Road Functions causes early convergence, specifically the GA gets stuck in local optima created by the stepping stones. While the authors offer an indirect insight into BBH, it is still unclear what exactly happens to schemata in the course of a genetic algorithm. The `building blocks' in Mitchell et al.'s paper are somewhat artificial as they are defined in the fitness function rather than discovered by the population itself. It is not obvious if the same findings would apply to `real' building blocks found by a GA. In practical terms, the most common application of the BBH is as a heuristic in the design of efficient encodings for GAs. In particular, encodings are often chosen as to allow building blocks to be combined meaningfully by the genetic operators \cite{janikow1991experimental}.  
\\
\\
We believe the fundamental problem with using the BBH as a narrative for GAs (and with seeing GAs as schema processors to begin with for that matter) is that there is no method to {\it observe} schemata being manipulated by a GA, as such it is hard, if not impossible to test accurately any meaningful statement about the type of schema processing performed by a GA.  Thus, in the current work, we present a natural method for identifying the set of schemata being tested by a population. We call this method the `schematic completion'. We also find that the schemata found by the schematic completion always forms a mathematical structure called a complete lattice, we call `the schematic lattice'. Using these methods, one can observe the {\it exact} schema processing performed by a GA by simply calculating the schemata present in each population and thus. These methods, we hope, are useful tools for studying GAs through the conceptual lens of schema processing and we hope that they will deepen the understanding of GAs. Specifically, we hope to be able to explore what makes a function difficult or easy for a GA to optimize, understand how useful it is to see GAs as `schema processors' and perhaps inform the design of better selection, crossover and mutation methods as to improve schema processing of a GA.
\\
\\
In the text that follows firstly the basics of order and lattice theory are introduced, this theory will be used to define our method for calculating the schemata present in a populating. Secondly, schemata are formally defined and the notions of schematic completion and the schematic lattice are introduced. To demonstrate the usefulness of these methods, in section 4 and 5, we show how these novel notion can be used to study GAs both theoretically and experimentally. In section 4, theoretically we show the search space for good schemata is a complete lattice and that each generation samples a complete sublattice of this search space. We also find that combining schemata is not a good method to explore the search space of schemata as in most cases only a subset of the search space can be reached by combining schemata alone. In section 5, we experimentally examine how well various crossover methods combine building blocks. We find that only 25-35\% of building blocks in a generation result from the combination of the previous generation's building blocks. We also find that an increase in building block combinations does not correspond to increase in the efficiency of the GA. In the appendix of this article, an open source Python package called schematax is introduced, which efficiently calculates the schematic completion and to draws the schematic lattice. We encourage readers interested in the following work to exploit this package for their research into GAs.
\\
\\
It should be noted that the mathematical insights presented in this article are not particularly difficult to understand in and of themselves. However to situate schemata in the broader mathematics of order and lattice theory some basic mathematical definitions and proofs are required. If the reader is not acquainted with this subject area we advise them to look simply at the examples and figures in section 3 to intuit the notions of the schematic completion and schematic lattice.  


\section{Background Mathematics}
In this section, we cover the background mathematics required to introduce the basic insights into schemata presented in the next section. Many of the following definitions regarding order theory and lattice theory are adapted from Ganter and Wille's book on Formal Concept Analysis \cite{ganter1999formal} and also Birkhoff's seminal book on lattice and order theory \cite{birkhoff1948lattice}. If these areas are understood, please skip ahead to the next section.
Firstly we cover Order and Lattice theory, secondly, we introduce Closure Systems and Galois Connections. 
\subsection{Order and Lattice theory}
     
\begin{definition}
A relation $\leq$ is called a partial order on a set $S$ if for all $a,b,c \in S$ it satisfies:
\begin{enumerate}
\item reflexivity: $a \leq a$               
\item antisymmetry: $a \leq b$ and $b \leq a \implies a = b$
\item transitivity: $a \leq b$ and $b \leq c \implies a \leq c$
\end{enumerate}
A {\bf partially ordered set} (poset for short) is a pair $(S,\leq)$ with $\leq$ being a partial order on the set $S$. We use $a < b$ if $a \leq b$ and 
$a \neq b$. $\geq$ denotes the inverse of $\leq$.
\end{definition}
\begin{definition}
A {\bf lower neighbour} of an element $b$ is another element $a$ such that there is no element $c$ with: $a < c < b$. In this case, $b$ is an {\bf upper neighbour} of a and we write $a \prec b$ to indicate this.  In the literature, it is also said as  $b$ {\bf covers} $a$.  
\end{definition}
Every finite partially ordered set, $(S,\leq)$ can be represented by a {\bf Hasse diagram}. Each element in $S$ is depicted by a circle. For any $a,b \in S$ if $a \prec b$ a line is drawn between the circles representing $a$ and $b$. Using a Hasse diagram one can read off any order relation: $a < b$ iff there is a descending path from $b$ to $a$. Figure 1 shows two Hasse Diagrams.
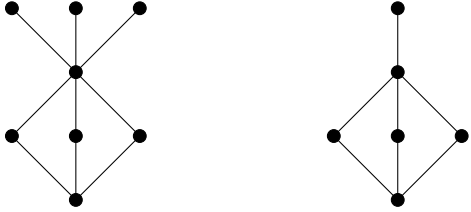
\begin{figure}[h!]
\centering
    \begin{tikzpicture}[scale=.85]

    \draw(0,0)[fill] circle (.1cm) node[right] {};
        \draw (0,0) -- (-1,1);
        \draw (0,0) -- (0,1);
        \draw (0,0) -- (1,1);
    \draw(0,1)[fill]  circle (.1cm) node[right] {};
        \draw (0,1) -- (0,2);
    \draw(-1,1)[fill]  circle (.1cm) node[right] {};
        \draw (-1,1) -- (0,2);
    \draw(1,1)[fill]  circle (.1cm) node[right] {};
        \draw (1,1) -- (0,2);
    \draw(0,2)[fill]  circle (.1cm) node[right] {};
        \draw (0,2) -- (0,3);
        \draw (0,2) -- (1,3);
        \draw (0,2) -- (-1,3);
    \draw(0,3)[fill]  circle (.1cm) node[right] {};
    \draw(1,3)[fill]  circle (.1cm) node[right] {};
    \draw(-1,3)[fill]  circle (.1cm) node[right] {};
    \end{tikzpicture}
    \hspace{2cm}
        \begin{tikzpicture}[scale=.85]

    \draw(0,0)[fill] circle (.1cm) node[right] {};
        \draw (0,0) -- (-1,1);
        \draw (0,0) -- (0,1);
        \draw (0,0) -- (1,1);
    \draw(0,1)[fill]  circle (.1cm) node[right] {};
        \draw (0,1) -- (0,2);
    \draw(-1,1)[fill]  circle (.1cm) node[right] {};
        \draw (-1,1) -- (0,2);
    \draw(1,1)[fill]  circle (.1cm) node[right] {};
        \draw (1,1) -- (0,2);
    \draw(0,2)[fill]  circle (.1cm) node[right] {};
        \draw (0,2) -- (0,3);
    \draw(0,3)[fill]  circle (.1cm) node[right] {};
    \end{tikzpicture}
    
\caption{Two Hasse diagrams with 8 elements}
\end{figure}
\begin{definition}
A {\bf rank function}, $\rho$, over a poset $(M,\leq)$ is a function which maps each element in $M$ to the natural numbers such that for $x,y \in M$ the following two properties are satisfied:
$$x \leq y \implies \rho(x) \leq \rho(y)$$
$$ x \preceq y \implies \rho(y) = \rho(x) +1$$
\end{definition}
\begin{definition}
Let $(M,\leq)$ be a poset and let $A \subseteq M$. A {\bf lower bound} of $A$ is an element $m \in M$ with $m \leq a$ for all $a \in A$. An {\bf upper bound} of $A$ can be defined dually. If there is a largest element in the set of all lower bounds of $A$, this element is called the  {\bf infimum} of $A$, denoted $\bigwedge A$. Dually, if there is a smallest element in the set of all upper bounds of $A$, this element is called the {\bf supremum} of $A$, denoted $\bigvee A$. If $A = \{x,y\}$, the infimum of $A$ is called the join and is denoted $x \wedge y$, dually the supremum of $A$ is called the meet and is denoted $x \vee y$
\end{definition}
Intuitively, one can think of the supremum of a $A$ as the ``smallest element in $M$ which is greater than or equal to all elements in $A$". The infimum of $A$ can be seen as ``the largest element in $M$ which is less than or equal to all elements in $A$". 
\begin{definition}
We call an ordered set $\boldsymbol{L} := (L,\leq)$ a {\bf lattice} if for every $x,y \in L$, $x \wedge y$ and $x \vee y$ exist.  We call $(L,\leq)$ a {\bf complete lattice} if for every $X \subseteq L$, the infimum $\bigwedge X$ and the supremum $\bigvee X$ always exist. Every complete lattice, $\boldsymbol{L}$ has a largest element $\bigvee L$ called the {\bf unit element} of $\boldsymbol{L}$, denoted $\boldsymbol{1}_L$. Dually it has a smallest element, $\bigwedge L$, called the  {\bf zero element} of $\boldsymbol{L}$, denoted $\boldsymbol{0}_L$.  
\end{definition}
Every complete lattice is of course a lattice. Moreover, every finite non-empty lattice is a complete lattice.
\begin{example}
The left Hasse Diagram in figure 1 is not a lattice nor a complete lattice as the join on the top most elements does not exist. However, the right Hasse diagram in figure 1 is a complete lattice (and thus a lattice) as the supremum and infimum exist for any subset of elements.  Any closed real interval $[x,y]$ with $\leq$ under its normal interpretation as an ordering is a complete lattice. However, any unbounded set of real numbers is not a complete lattice, but is a lattice. The power set of any non-empty set with $\subseteq$ as an ordering is an exemplary complete lattice. The Hasse diagram of the complete lattice formed by  $(\mathcal{P}(\{a,b,c\}), \subseteq)$ is shown in figure 2.

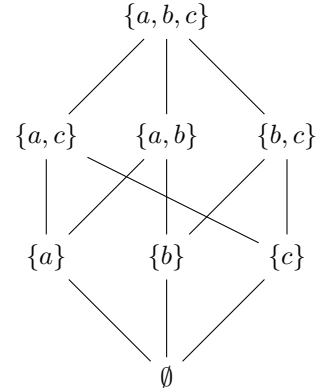
\begin{figure}[h!]
\centering
\begin{tikzpicture}[node distance=1.6cm]
\node(0)                           {$\{a,b,c\}$};
\node(2)      [below of =0]  {$\{a,b\}$};
\node(3)      [left of=2] {$\{a,c\}$};
\node(1)       [right of=2] {$\{b,c\}$};

\node(4)     [below of=3]          {$\{a\}$};
\node(5)    [below of=2]            {$\{b\}$};
\node(6)    [below of=1]            {$\{c\}$};
\node(7)    [below of=5]     {$\emptyset$};

\draw(0)       -- (1);
\draw(0)       -- (2);
\draw(0)       -- (3);

\draw(2)       -- (4);
\draw(2)       -- (5);

\draw(3)        -- (4);
\draw(3)        -- (6);

\draw(1)        -- (6);
\draw(1)        -- (5);

\draw(4)        -- (7);
\draw(5)        -- (7);
\draw(6)        -- (7);
\end{tikzpicture}
\caption{The complete lattice formed by $(\mathcal{P}(\{a,b,c\}), \subseteq)$.}
\end{figure}
\end{example}
\begin{definition}
A subset $A$ of a complete lattice $\boldsymbol{L}$ which is closed under suprema and infima, specifically:
$$ X \subseteq A  \implies \bigvee X,\bigwedge X \in A$$
is called {\bf complete sublattice}. If $A$ is only closed under suprema it is called a {\bf meet-subsemilattice}. Dually, if $A$ is closed under infima only, it is called a {\bf join-subsemilattice}. 
\end{definition}
\begin{example}
The poset $(\{\{a,b,c\},\{a,b\},\{b,c\},\{b\}\},\subseteq)$ is a complete sublattice of the complete lattice defined in figure 2. This complete sublattice can be seen in figure 3 below:
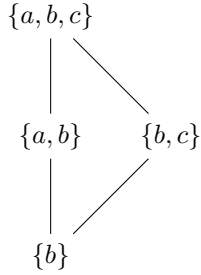
\begin{figure}[h!]
\centering
\begin{tikzpicture}[node distance=1.6cm]
\node(0)                           {$\{a,b,c\}$};
\node(2)      [below of =0]  {$\{a,b\}$};
\node(1)       [right of=2] {$\{b,c\}$};

\node(5)    [below of=2]            {$\{b\}$};

\draw(0)       -- (1);
\draw(0)       -- (2);

\draw(2)       -- (5);

\draw(1)        -- (5);

\end{tikzpicture}
\caption{The complete lattice $(\{\{a,b,c\},\{a,b\},\{b,c\},\{b\}\},\subseteq)$. This is a complete sublattice of the complete lattice $(\mathcal{P}(\{a,b,c\}),\subseteq)$ shown in figure 2.}
\end{figure}
\end{example}

\begin{definition}
Let $(L,\leq)$ be a complete lattice. An element $a \in L$ is called an {\bf atom} of L if $\boldsymbol{0}_L < a$ and there does not exist an element $x \in L$ such that  $\boldsymbol{0}_L < x < a$. L is called {\bf atomic} if every element $b > \boldsymbol{0}_L $ implies that $b$ is an atom or that $b$ has an atom below it. That is,  $b \geq a > 0$. L is called {\bf atomistic} if every element in $L$ can be given by the supremum of a subset of the atoms.
\end{definition}

\begin{example}
The set of atoms of the complete lattice shown in figure 2 is: $\{\{a\}, \{b\}, \{c\}\}$. This complete lattice is atomic and atomistic. The complete lattice shown on the right of figure 1 is atomic, however it is not atomistic as the top element cannot be reached by the supremum on any subset of the atoms.  
\end{example}

\subsection{Closure Systems and (monotone) Galois Connections }
\begin{definition}
A {\bf closure system} on a set $S$ is a set of subsets of $S$ which contains $S$ and is closed under intersection. Specifically, $A \subseteq \mathcal{P(S)}$ is called a closure system on $S$ if $S \in A$ and:
$$X \subseteq A \implies \bigcap X \in A $$
\end{definition}
\begin{example}
Consider, $S = \{a,b,c\}$ and let $ A = \{\{a,b,c\}, \{a,b\},\{a\}\}$. $A$ is a closure system on $S$, as $S$ is included in $A$ and $A$ is closed under intersection as any intersection of elements in $A$ is also a member of $A$.
\end{example}
\begin{definition}
A {\bf closure operator} on a set $S$ is map, $cl: \mathcal{P}(S) \mapsto \mathcal{P}(S)$ which satisfies the following for all $X,Y \subseteq S$
\begin{enumerate}
\item extensity: $X \subseteq cl(X)$.
\item monotonicity: $X \subseteq Y \implies cl(X) \subseteq cl(Y)$.
\item idempotency: $cl(cl(X)) = cl(X)$.
\end{enumerate}
\end{definition}
Closure operators and closure systems are closely linked, as can be seen in the following theorem.
\begin{theorem}
If $cl$ is a closure operator on a set $S$ then the set:
$$ A_{cl} := \{cl(X)| X \subseteq S\}$$
(the set of all closures of a closure operator) is a closure system. Conversely if $A$ is a closure system on $S$ then the following operator:
$$cl_{A}(X) := \bigcap \{a \in A| X \subseteq a\}$$
Defines a closure operator. 
\end{theorem}
There is a bijection between closure operators and closure systems. Every closure operator has a corresponding closure system and every closure system has a corresponding closure operator. A closure system can be seen as the set of all closures of a closure operator. 
Whats more, closure systems (and thus closure operators) are closely linked with complete lattices, as will be seen in the next proposition.  
\begin{proposition}
If $A$ is a closure system then $(A,\subseteq)$ is a complete lattice, where for $X \subseteq A$ the infinum, $\bigwedge X$ is given by $\bigcap X$ and the supremum $\bigvee X$ is given by $cl_{A}(\bigcup X)$. Every complete lattice is isomorphic to the lattice of all closures of a closure system. 
\end{proposition}
\begin{example}
The complete lattice formed by the closure system $A$ appearing in the example above is seen in figure 4, below:
\begin{figure}[h!]
\centering
\begin{tikzpicture}[node distance=1.2cm]
\node(0)                           {$\{a,b,c\}$};
\node(1)      [below of =0]  {$\{a,b\}$};
\node(2)       [below of=1] {$\{a\}$};
\node(3)        [below of =2] {$\emptyset$};

\draw(0)       -- (1);
\draw(1)       -- (2);
\draw(2)        -- (3);

\end{tikzpicture}
\caption{The complete lattice formed by the closure system $A$ with $\subseteq$ as an ordering.}
\label{fig:1}
\end{figure}
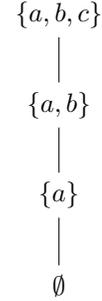
\end{example}
\begin{definition}
Suppose we have two partially ordered sets, $(A, \leq)$ and $(B, \leq)$.  Two montone functions over these sets, $F: A \mapsto B$ and $G: B \mapsto A$ are called a {\bf Galois connection} of $A$ and $B$ if we have for all $a \in A$ and $b \in B$:
$$F(a) \leq b \iff a \leq G(b)$$
In this case, $F$ is called the {\bf lower adjoint} and G the {\bf upper adjoint}.
Equivalently, if $F$ and $G$ satisfy the following conditions they also form a Galois Connection. For all $a_1,a_2 \in A$ and all for call $b_1 \in B$ we have:
\begin{enumerate}
\item $a_1 \leq a_2 \implies F(a_1) \leq  F(a_2)$ 
\item $a_1 \leq GF(a_1)$ 
\item $a_1 \leq G(b_1) \implies F(a_1) \leq b_1 $
\end{enumerate}
\end{definition}
\begin{proposition}
The composition $GF: A \mapsto A$ is a closure operator. 
\end{proposition}

\FloatBarrier
\section{Schemata}
In this section we define schemata and define two basic operations, the {\bf expansion} and {\bf compression}. Secondly, using these operators we define the notions of the schematic completion and the schematic lattice and prove properties about these notions. 
\\
Let $\Sigma$ be a finite alphabet which does not contain symbol $*$. We use $\Sigma^l$ to denote the set of all words of length $l$ over $\Sigma$.
\begin{definition}
The {\bf schematic alphabet} of $\Sigma$ is $\Sigma$ with an extra symbol, $*$, the {\bf wild card} symbol. We use $\Sigma_*$ to denote the schematic alphabet of $\Sigma$. Symbols in $\Sigma_*$ which are not the wild card symbol are called {\bf fixed} symbols.
\end{definition}
\begin{definition}
A {\bf schema} is a word over $\Sigma_*$. We use $\Sigma^l_*$ to denote all schemata of length $l$ over $\Sigma_*$ including the empty schema, $\epsilon_*$. 
\end{definition}
\begin{example}
Let $\Sigma$ be the binary alphabet, that is $\{1,0\}$. The schematic alphabet of $\Sigma$, denoted $\Sigma_*$, is the alphabet $\{1,0,*\}$. An example of a schema in $\Sigma^3_*$ is $1**$. 
\end{example}
\begin{definition}
For any schema $s \in \Sigma^l_*$ we define the following operator ${\uparrow}s$, called the {\bf expansion} of $s$, which maps $s$ to a subset of words in $\Sigma^l$:
$${\uparrow}s := \{b \in \Sigma^l | b_i = s_i \mbox{ or } s_i = * \mbox{ for each }  i \in \{1,...,l\}\}$$
where subscript $i$ denotes the character at position $i$ in a word or schema.  When $s = \epsilon_*$ then ${\uparrow}s  = \emptyset$.  More simply put, ${\uparrow}s$ is the set of all words in $\Sigma^l$ that can be made by exchanging the $*$ symbols in $s$ with symbols from $\Sigma$.
\end{definition}
\begin{example}
Continuing the example above, ${\uparrow}1** = \{100,110,101,111\}$. The $1$ in the first position is fixed. Note ${\uparrow}111 = \{111\}$ and ${\uparrow}\epsilon_* = \emptyset$. 
\end{example}
\begin{definition}
Conversely, for any $A \subseteq \Sigma^l$ we define ${\downarrow}{A}$, called the {\bf compression} of $A$,  which maps $A$ on to a schema $s \in \Sigma_*^l$
$${\downarrow}A:= s$$
where $s$ is a schema of length $l$ such that the symbol at position $i$ in $s$ is determined in the following way: if $x_i = y_i$ for all $x,y \in A$ then $s_i = x_i$ otherwise $s_i = *$. If $A = \emptyset$ then ${\downarrow}A = \epsilon_*$. One can think of this operator as stacking up all the items in $A$ and if all elements in a column are equivalent, the symbol at that position in $s$ takes this value, otherwise there is a wild card symbol. 
\end{definition}
\begin{example}
Let $A = \{100,000,010\}$ then ${\downarrow}A = **0$. Note if $A = \emptyset$ then ${\downarrow}A = \epsilon_*$. If $A = \{100\}$ then ${\downarrow}A = 100$ 
\end{example}
\begin{definition}
Schemata can be ordered. For any $a,b \in \Sigma^l_*$ we say $a \leq b$ if and only if ${\uparrow}a \subseteq {\uparrow}b$. It follows that $\leq$ is a partial ordering on a set of schemata from the reflexivity, antisymmetry and transitivity of the subset relation.    
\end{definition}
\begin{example}
Again let $\Sigma = \{1,0\}$. Consider the following schema in $\Sigma^2_*$: $\epsilon_*$, $11$, $1*$, $**$. They are ordered in the following way: $\epsilon_* \leq 11 \leq 1* \leq **$.
This is because ${\uparrow}\epsilon_* \subseteq {\uparrow}11 \subseteq {\uparrow}1* \subseteq {\uparrow}** = \emptyset \subseteq \{11\} \subseteq \{11,10\} \subseteq \{11,10,01,00\}$.
\end{example}
\begin{definition}
It is possible to define compression in terms of expansion:
$${\downarrow}A:= s$$
such that $A \subseteq {\uparrow}s$ and for any $r \in \Sigma^l_*$
$$A \subseteq {\uparrow}r \implies s \leq r$$
That is, ${\downarrow}A$ is the schema whose expansion includes $A$ and is the smallest such schema to do so. 
\end{definition}
\begin{definition}
Conversely we can define expansion in terms of compression:
$${\uparrow}s:= A$$
such that ${\downarrow}A = s$ and for any $B \subseteq \Sigma^l$ we have:
$$ {\downarrow}B = s \implies B \subseteq A $$
That is, ${\uparrow}s$ is the largest subset of words whose compression is equal to $s$.
\end{definition}
We will soon see that definitions 3 and 4 of the expansion and compression operators are useful computationally while definitions 6 and 7 are useful in proving properties about schemata.  
\begin{proposition}
For any schema $s \in \Sigma^l_*$, we have ${\downarrow}{\uparrow}s = s$.
\end{proposition}
\begin{proof}
Let $A = {\uparrow}s$ definition III.6 trivially yields ${\downarrow}A = s$, thus  ${\downarrow}{\uparrow}s = s$.
\end{proof}
\begin{proposition}
For $A \subseteq \Sigma^l$, we have $A \subseteq {\uparrow}{\downarrow}A$.
\end{proposition}
\begin{proof}
Let $s = {\downarrow}A$, definition III.7 trivially yields $A \subseteq {\uparrow}s$, thus $A \subseteq {\uparrow}{\downarrow}A$.
\end{proof}
\begin{proposition}
Compression is monotonic, that is for $A,B \subseteq \Sigma^l$:
$$A \subseteq B \implies {\downarrow}A \leq {\downarrow}B$$
\end{proposition}
\begin{proof}
Assume $A \subseteq B$ we will show ${\downarrow}A \leq {\downarrow}B$. Proposition III.6 gives $B \subseteq {\uparrow}{\downarrow}B$. Since $A \subseteq B \subseteq {\uparrow}{\downarrow}B$ the transitivity of the subset relation yields $A \subseteq {\uparrow}{\downarrow}B$. Let ${\downarrow}A = a$ and  ${\downarrow}B = b$. Since we have $A \subseteq {\uparrow}b$, definition III.6 applied to ${\downarrow}A$ yields $a \leq b$. Thus we have ${\downarrow}A \leq {\downarrow}B$.
\end{proof}
\begin{proposition}
For $A \subseteq \Sigma^l$ and $b \in \Sigma^l_*$ we have:
$$A \subseteq {\uparrow}b \implies {\downarrow}A \leq b$$
\end{proposition}
\begin{proof}
We will assume $A \subseteq {\uparrow}b$ and show  ${\downarrow}A \leq b$. Proposition III.7 gives us, ${\downarrow}A \leq {\downarrow}{\uparrow}b$. Proposition III.5 yields ${\downarrow}A \leq b$.  
\end{proof}

\begin{lemma}
The compression and expansion operators form a Galois connection, where $\downarrow$ is the lower adjoint and $\uparrow$ the upper adjoint. 
\end{lemma}
\begin{proof}
Let $A,B \subseteq \Sigma^l$ and $s \in \Sigma^l_*$. From definition II.9, it is sufficient to show:
\begin{enumerate}
\item $A \subseteq B \implies {\downarrow}A \leq  {\downarrow}B $ 
\item $A \subseteq {\uparrow}{\downarrow}A$ 
\item $A \subseteq {\uparrow}s \implies {\downarrow}A \leq s $
\end{enumerate}
1) was shown in proposition III.7, 2) was shown in III.6 and 3) in proposition III.8 .
\end{proof}
\begin{definition}
For a set $A \subseteq \Sigma^l$, we call the process of calculating the compression on each subset of A, that is $\{{\downarrow}X | X \subseteq A\}$, the {\bf schematic completion} of $A$, denoted $\mathcal{S}(A)$. 
\end{definition}
\begin{example}
Let $\Sigma^l = \{1,0\}^3$ and $A = \{110, 100, 001, 000\}$ the schematic completion of $A$,  results in the following set: 
$$\{001, 100, 000, 110, 00*, *00, 1*0, **0, *0*, ***, \epsilon_*\}$$
For example, the schema $*00$ comes from the compression on the subset $\{110, 000\}$. 
\end{example}
\begin{theorem}
{\bf(The fundamental theorem of schemata)} For any $A \subseteq \Sigma^l$, the schematic completion of $A$, $\mathcal{S}(A)$ ordered by $\leq$ forms a complete lattice, that is the poset $(\mathcal{S}(A), \leq)$ is a complete lattice. We call this lattice the {\bf schematic lattice} of $A$. Let $X \subseteq \mathcal{S}(A)$, the supremum, ${\bigvee}X$, is given by ${\downarrow}(\bigcup\limits_{s \in X}{{\uparrow}s})$. The infimum, ${\bigwedge}Y$, is given by ${\downarrow}\{ a \in  \bigcap\limits_{s \in X}{\uparrow}s | a \in A\}$.    
\end{theorem}
\begin{proof}
Lemma III.9 tells us $\downarrow$ and $\uparrow$ form a Galois connection, where $\downarrow$ is the lower adjoint and $\uparrow$ is the upper adjoint. As such, proposition II.8 yields ${\uparrow}{\downarrow}$ as a closure operator. Hence, from Theorem II.5,  $\{{\uparrow}{\downarrow} X| X \subseteq A\}$ is a closure system. Thus, the poset $(\{{\uparrow}{\downarrow} X| X \subseteq A\}, \subseteq)$ forms a complete Lattice (proposition II.6). Proposition II.6 also yields for $Y \subseteq \{{\uparrow}{\downarrow} X| X \subseteq A\}$, the supremum is given by ${\uparrow}{\downarrow}\bigcup{Y}$ and the infimum, $\bigcap{Y}$.  From the definition of ordering over schemata, we then have $(\{{\downarrow} X| X \subseteq A\}, \leq)$ as a complete lattice and the infimum, ${\bigwedge}Y$, is given by ${\downarrow}\{ a \in  \bigcap\limits_{s \in X}{\uparrow}s | a \in A\}$.  
\end{proof}

It is easy to check that the atoms of the schematic lattice $(\mathcal{S}(A), \leq)$ is the set $A$ and $(\mathcal{S}(A), \leq)$ is atomistic.
\begin{example}
Continuing the above example, the schematic lattice formed from schematic completion on A can be seen in figure 5.

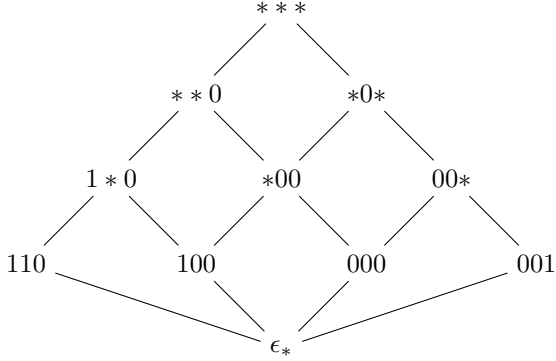
\begin{figure}[h!]
\centering
\begin{tikzpicture}[node distance=1.6cm]
\node(0)                           {$***$};
\node(1)      [below left of=0]  {$**0$};
\node(2)      [below right of= 0] {$*0*$};

\node(3)       [below left of =1] {$1*0$};
\node(4)       [below right of =1] {$*00$};
\node(5)       [below right of =2] {$00*$};

\node(6)       [below left of =3] {$110$};
\node(7)       [below right of =3] {$100$};
\node(8)       [below right of =4] {$000$};
\node(9)       [below right of =5] {$001$};

\node(10)     [below left of =8] {$\epsilon_*$};

\draw(0)       -- (1);
\draw(0)       -- (2);

\draw(1)       -- (3);
\draw(1)       -- (4);

\draw(2)       -- (4);
\draw(2)       -- (5);

\draw(3)       -- (6);
\draw(3)       -- (7);
\draw(4)       -- (7);
\draw(4)       -- (8);
\draw(5)       -- (8);
\draw(5)       -- (9);

\draw(6)       -- (10);
\draw(7)       -- (10);
\draw(8)       -- (10);
\draw(9)       -- (10);
\end{tikzpicture}
\caption{The schematic lattice formed by the schematic completion on the set $A = \{110, 100, 001, 000\}$ ordered by $\leq$, that is the complete lattice $(\mathcal{S}(A),\leq)$}
\end{figure}
\end{example}
\begin{example}
Of course, the schematic completion is not restricted to words over the binary alphabet. Let $\Sigma= \{a,b,c,\dots, z\}$ and consider the set $X \subseteq \Sigma^4$:
$$X := \{help,kelp,yell,tell,talk,walk\}$$
The schematic completion of this set gives us the schematic lattice in figure 6. 
\begin{figure}[htp]
\centering
\begin{tikzpicture}[node distance=1.4cm]
\node(0)                           {$**l*$};

\node(1)      [below left of=0]  {$*el*$};

\node(2)      [below left of=1]  {$*elp$};
\node(3)      [right of=2]  {$*ell$};
\node(4)      [right of=3]  {$t*l*$};
\node(5)      [right of=4]  {$*alk$};

\node(6)      [below left of=2]  {$help$};
\node(7)      [right of=6]  {$kelp$};
\node(8)      [right of=7]  {$yell$};
\node(9)      [right of=8]  {$tell$};
\node(10)     [right of=9]  {$talk$};
\node(11)      [right of=10]  {$walk$};
book.com/
\node(12)      [below of=8]  {$\epsilon_*$};

\draw(0)       -- (1);
\draw(0)       -- (4);
\draw(0)       -- (5);
\draw(1)       -- (2);
\draw(1)       -- (3);
\draw(2)       -- (6);
\draw(2)       -- (7);
\draw(3)       -- (8);
\draw(3)        -- (9);
\draw(4)        -- (9);
\draw(4)        -- (10);
\draw(5)        -- (10);
\draw(5)        -- (11);

\draw(12)        -- (6);
\draw(12)        -- (7);
\draw(12)        -- (8);
\draw(12)        -- (9);
\draw(12)        -- (10);
\draw(12)        -- (11);
\end{tikzpicture}
\caption{The complete lattice formed from the schematic completion on $X$, that is $(\mathcal{S}(X),\leq)$.  Notice, as each word in $X$ has an $l$ in the 3rd position the unit element of this lattice is $**l*$. }
\label{fig:4}
\end{figure}
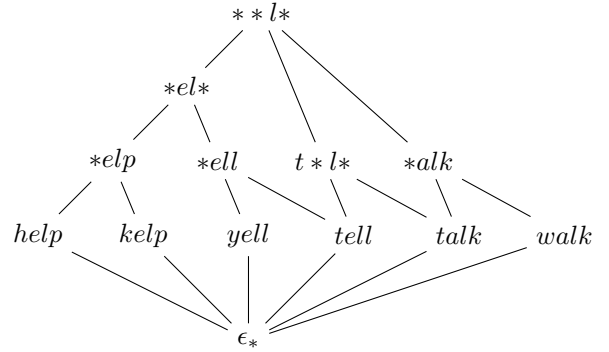
\end{example}
We can more precisely define some of the original properties of schema using the above definitions.  
\begin{definition}
The {\bf order} (not to be confused with partial order) of schema $s \in \Sigma^l_*$ is the number of fixed symbols in $s$ and is denoted $\mathbf{o}(s)$. The order of $s$ can be equivalently defined as:
$$\mathbf{o}(s) := l - \log_{|\Sigma|}(|{\uparrow}s|)$$
Similarly the {\bf antiorder}, denoted $\mathbf{o'}$ of $s$ is the number of wild card symbols in $s$, which can be defined as:
$$\mathbf{o'}(s) := \log_{|\Sigma|}(|{\uparrow}s|)$$
\end{definition}
\begin{example}
Let s = $11**$. We can count the number of fixed symbols to give us $\mathbf{o}(s) = 2$. Equivalently $l - log_{|\Sigma|}(|{\uparrow}s|) = 4- log_{2}(|\{1100,1110,1101,1111\}|) = 4 - log_{2}(4) = 2$.  
\end{example}
\begin{proposition}
If we set, $\mathbf{o'}(\epsilon_*) = -1$ then $\mathbf{o'}$ is a rank function over schemata. 
\end{proposition}
\begin{proof}
To show $\mathbf{o'}$ is a rank function over schemata, it is sufficient to show for $x,y \in \Sigma^l_*$:
$$\mbox{1.) } x \leq y \implies \mathbf{o'}(x) \leq \mathbf{o'}(y)$$
$$\mbox{2.) } x \preceq y \implies \mathbf{o'}(y) = \mathbf{o'}(x) +1$$
First we will show 1). Assume $x \leq y$, using the definition of ordering we have ${\uparrow}x \subseteq {\uparrow}y$. It follows then that $|{\uparrow}x| \leq |{\uparrow}y|$, and thus $\mathbf{o'}(x) \leq \mathbf{o'}(y)$.
Now we show 2). Assume $x \preceq y$ we then have ${\uparrow}x \subset {\uparrow}y$ with $|{\uparrow}y| = |{\uparrow}x| + |\Sigma|$, thus $\mathbf{o'}(y) = \mathbf{o'}(x) +1$. 
\end{proof}
\begin{corollary}
Using the same method, it is possible to show that the order of a schema is a dual rank function over schemata if we make $\mathbf{o}(\epsilon_*) = l+1$. That is for  $x,y \in \Sigma^l_*$: 
$$\mbox{1.) } x \leq y \implies \mathbf{o}(x) \geq \mathbf{o}(y)$$
$$\mbox{2.) } x \preceq y \implies \mathbf{o}(x) = \mathbf{o}(y) +1$$
\end{corollary}
\begin{definition}
A word $a \in \Sigma^l$ is said to be an {\bf instance} of schema $s$ if and only if $a \in {\uparrow}s$.
\end{definition}
\begin{example}
The word $111$ is an instance of the schema $1**$ as $111 \in {\uparrow}1** = 111 \in \{111,110,101,100\}$  
\end{example}
We now introduce some novel properties of schemata not originally described in previous works.
\begin{definition}
For some $A \subseteq \Sigma^l$, the {\bf confidence} of a schema $s \in \mathcal{S}(A)$ is given as:
$$ \frac{|\{a\in {\uparrow}s | a \in A  \}|}{|{\uparrow}s|}$$
In more simple terms, the confidence of $s \in \mathcal{S}(A)$ is the proportion of ${\uparrow}s$ that is found in $A$.
\end{definition}
The confidence of a schema is useful in GAs for understanding how confident one can be in the fitness assigned to a schema. In particular, the more instances of a schema has, the more we can trust it's fitness.
\\
\\
The following lemmas and proposition are useful for computations over schemata as they allow us to determine the ordering of schemata by considering only the characters and wild cards rather than having to compute the expansion explicitly. 
\begin{lemma}
Let $A,B \subseteq \Sigma^l$,  $a = {\downarrow}A$ and $b = {\downarrow}B$.  $A \subseteq B$ if and only if for all $i$ we have:
$$a_i = b_i \text{ or } b_i = *$$
\end{lemma}
\begin{proof}
This follows directly from the definition of compression. Let $C = B\setminus A$. If $C$ is empty, then clearly $a_i = b_i$ for all $i$, otherwise, for any $x \in C$ and any $i$, $x_i$ can either differ from $a_i$, thus ${b_i} = *$ or $x_i$ can equal $a_i$, thus $b_i = a_i$.  
\end{proof}
\begin{proposition}
For $a,b \in \Sigma^l_*$, $a \leq b$ if and only if for all $i$
$$a_i = b_i \text{ or } b_i = *$$
\end{proposition}
\begin{proof}
Let $a,b$ be any schema in $\Sigma_*^l$.\\
($\implies$) We will assume $a \leq b$ and show $a_i = b_i \text{ or } b_i = *$ for all $i$. 
From the definition of order over schema we have ${\uparrow}a \subseteq {\uparrow}b$, let $a' = {\downarrow}{\uparrow}a$, $b' = {\downarrow}{\uparrow}b$, since ${\uparrow}a \subseteq {\uparrow}b$ lemma III.18 yields
$$a'_i = b'_i \text{ or } b'= *$$
for all $i$.
Proposition III.5 gives us  ${\downarrow}{\uparrow}a = a$ and ${\downarrow}{\uparrow}b=b$, meaning $a = a'$ and $b=b'$, thus:
$$a_i = b_i \text{ or } b_i = *$$
for all i.\\
($\impliedby$) We will assume $(a_i = b_i$ or $b_i = * \text{ for all i})$  and show $a \leq b$. From Lemma III.18, there exist $A,B \subset \Sigma^l$ with ${\downarrow}A = a$ and ${\downarrow}B = b$ such that $A \subseteq B$.\\
Proposition III.6 yields: $B \subseteq {\uparrow}b$. Since we have $A \subseteq B$ and $B \subseteq {\uparrow}b$, the transitivity of the subset relation yields  $A \subseteq {\uparrow}b$. As ${\downarrow}A = {\uparrow} a$, we have ${\uparrow}a \subseteq {\uparrow}b$.
\end{proof}
This concludes out order theoretical interpretation of schemata. However, for brevity's sake, many interesting many properties regarding schemata have not been mentioned here. For example a link to a area of mathematics called Formal Concept Analysis, which is concerned with finding concept hierarchies in object-feature relationships \cite{ganter1999formal}.  Secondly, as schemata are a (small) subset of regular expressions,  this makes the schematic completion akin to the induction of regular languages \cite{miller1957pattern, solomonoff1959new} where the search space is known to be a complete lattice \cite{dupont1994search}.
\\
\\
In the subsequent section it is shown how these simple insights into schemata, specifically the schematic completion and the schematic lattice, can be used to study how combining schemata explores the search space.

\section{How does combining schemata explore the search space?}
It is common to visualize the search space of GAs as a hypercube in which each schema defines a hyperplane \cite{whitley1994genetic}. 
In this section, we offer a complementary view based on the mathematics in the previous section. We show that the search space of schemata is a complete lattice and that each generation of a GA samples a complete sublattice of this complete lattice. Finally, we consider how combining schemata explores the search space.
\\
\\
For a GA working on binary strings (that is, $\Sigma =\{0,1\}$) of size $n$, the search space of all schemata is the set $\Sigma^n_*$. If we order this set the search space is revealed to be the complete lattice $(\Sigma^n_*, \leq)$. We could also construct the search space by using the schematic completion on the set of all possible words of length $n$, meaning: $$(\Sigma^n_*, \leq) = (\mathcal{S}(\Sigma^n), \leq)$$
Figure 8 shows the search space of schemata for GAs working on binary strings of length $3$. 
\\
\\
Given a generation $G_t$ at time $t$ of a GA, the schematic completion on $G_t$, $\mathcal{S}(G_t)$, yields at least a subset of the schemata being tested by $G_t$. However, it is still unclear if the schematic completion on $G_t$ yields {\it all} the schemata being tested by $G_t$. To explore this possible limitation, consider the population
$$G_t = \{1010,1111,1100,1000\},$$
The schematic completion on  $G_t$ returns the set: 
$$\{1111,1010,1100,1000, 1*1*,11**,10*0, 1*00, 1**0, 1***, \epsilon_*\}$$
Which forms the schematic lattice shown below in figure 7. 

\begin{figure}[h!]
\centering
\begin{tikzpicture}[node distance=1.6cm]
\node(0)                           {$1***$};
\node(1)      [below right of=0]  {$1**0$};

\node(2)      [below right of= 1] {$1*00$};
\node(3)      [left of= 2] {$10*0$};
\node(4)      [left of= 3] {$11**$};
\node(5)      [left of= 4] {$1*1*$};
\node(6)       [below of =2] {$1000$};
\node(7)       [left of = 6] {$1100$};
\node(8)       [left of =7] {$1010$};
\node(9)       [left of =8] {$1111$};
\node(10)     [below right of =8] {$\epsilon_*$};
\draw(0)       -- (1);
\draw(0)       -- (4);
\draw(0)       -- (5);

\draw(1)       -- (2);
\draw(1)       -- (3);

\draw(2)       -- (6);
\draw(2)       -- (7);

\draw(3)       -- (6);
\draw(3)       -- (8);
\draw(4)       -- (7);
\draw(4)       -- (9);
\draw(5)       -- (8);
\draw(5)       -- (9);

\draw(6)       -- (10);
\draw(7)       -- (10);
\draw(8)       -- (10);
\draw(9)       -- (10);
\end{tikzpicture}
\caption{The complete lattice formed from the schematic completion on the population $G_t$. That is, the lattice $(\mathcal{S}(G_t),\leq)$ }
\end{figure}
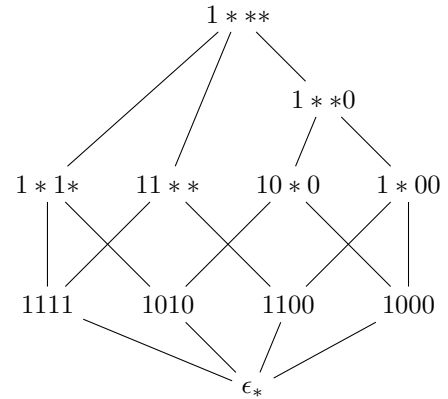
Is $\mathcal{S}(G_t)$ the set of {\it all} schemata being tested by population $G_t$? It is possible to argue that the schema $**00$ is being tested by $G_t$ as the individuals $1100$ and $1000$ both end in $00$ yet the schematic completion on $G_t$ does not yield $**00$. Thus it follows that the schematic completion does not return the complete set of all the schemata being tested by $G_t$. However, we can see that whenever an individual ends in $00$, it also begins with $1$. Thus $**00$ cannot be tested without having a $1$ in the beginning (as both $1100$ and $1000$ begin with a $1$), hence the schema which is being sampled by this population {\it is} $1*00$ which {\it does} appear in the schematic lattice. Indeed, as the fitness of a schema is given by the average of it's instances the fitness of $**00$ and $1*00$ have the same fitness, however $1*00$ is the accurate description of the schema being tested.  A similar argument can be made for any schema which appears to be omitted, thus, the schematic completion returns {\it all} schemata being sampled by $G_t$. Each generation of a GA then defines a set of schemata, $\mathcal{S}(G_t)$, which is the set of schemata being sampled by $G_t$. When $\mathcal{S}(G_t)$ is ordered by $\leq$, it is a complete sublattice of the search space $(\Sigma_*^n,\leq)$.
\\
\\
By calculating the schematic completion on the population for each generation, one can observe how schemata change during the course of a GA.
There are many natural experiments and questions which can be examined using this method (for example, one could test how well various schema theorems apply to a GA with a finite population size), however in this section we will focus on the more fundamental question of how combining schemata explores the search space.
\\
\\
It is proposed that a GA combines `good' schemata, however it is unknown how directly combining schemata explores the space of all schemata. To investigate, a method is required to identify when a schema results from the combination of a set of schemata. To do this we introduce the notion of {\bf schematic blending.}
\begin{definition}
For $X \subseteq \Sigma^l_*$ the {\bf schematic blending} of $X$, denoted $bl(X)$ is the schema given by 
$$bl(X) :={\downarrow}(\bigcap\limits_{s \in X}{\uparrow}s)$$
 If $bl(X)$ returns the empty schema then $X$ is said to be {\bf unblendable}, otherwise $X$ is said to be {\bf blendable}. Given a set of schemata $A$, the set of all schematic blends in $A$, that is $\{bl(x) | x \subseteq A\}$ is denoted $\mathcal{B}(A)$.
\end{definition}
\begin{example}
$bl(11***,**11*) = 1111*$, $11***$ and $01***$ are unblendable, $bl(1*10**,*110*1) = 110*1$.
\end{example}
It is clear that the order of the schematic blend of $A$, (if it is not empty) is greater than or equal to the order of any member of $A$. In addition, when performing the schematic blend on schemata which result from the schematic completion on a set of words, it is common for the blend  to result in schemata which already exist in schematic completion. Indeed, the following lemma explores this idea and is useful for understanding how combining schemata explores the schema space.
\begin{lemma}
{\bf (The schematic blending lemma)} For a set of schemata $A \subseteq \Sigma^l_*$, $bl(A)$ returns the largest schema $s \in \Sigma^l_*$ such that $s \leq s'$ for all $s'$ in $A$.
\end{lemma}
\begin{proof}
Let $A \subseteq \Sigma^l_*$. Then we have:
$$ bl(A) = {\downarrow}(\bigcap\limits_{s \in A}{\uparrow}s)$$
Let:
$$bl(A) = t = {\downarrow}({\uparrow}s_1 \cap {\uparrow}s_2 \cap \dots \cap{\uparrow}s_n)$$
Thus from the definition of compression we have $t$ as the largest schema with:
$${\uparrow}t \subseteq {\uparrow}s_1 \mbox{ and } {\uparrow}t \subseteq {\uparrow}s_2 \mbox{ and } {\uparrow}t \subseteq {\uparrow}s_3 \mbox{ and } \dots  \mbox{ and } {\uparrow}t \subseteq {\uparrow}s_n$$
and from the definition of partial ordering for schemata we have:
$$t \leq s_1 \mbox{ and } t \leq s_2 \mbox{ and } t \leq s_3 \mbox{ and } \dots  \mbox{ and } t \leq s_n $$
\end{proof}
In simple terms this lemma says: blending a set of schemata $A$ returns the largest schema $s \in \Sigma^l_*$ which is smaller than all elements in $A$. In this way the schematic blending is similar to the infimum operator over the schematic lattice. However, instead of returning a schema $s\in A$, a schema $s \in \Sigma^l_*$ is required.  
\\
\\
So far we understand this much: given a generation at time $t$,  $G_t$ of a genetic algorithm, the schemata being tested by this generation is the schematic completion of $G_t$, $\mathcal{S}(G_t)$, this samples a complete sub-lattice of the lattice all the possible schemata. The set of schemata which can be reached by blending these schemata is then $\mathcal{B}(\mathcal{S}(G_t))$. However, the schematic blending lemma tells us that blending only searches the spaces inbetween the layers of the lattice, and not `sideways' or upwards. What is more, schematic blending is idempotent that is: $\mathcal{B}(\mathcal{B}(\mathcal{S}(G_t)))  = \mathcal{B}(\mathcal{S}(G_t))$. Thus, by combining schemata alone, only a search over the lower neighbours of pre-existing schemata can be performed, meaning only a subset of the space of all schemata can be reached. Figure 8 demonstrates how blending explores the space of all schemata and is a good visual summary of the results from this section. Thus we conclude, blending schemata alone is not a good tool for exploring the space of all schemata.    
\begin{figure*}[t]
    \centering
    \includegraphics[width=0.99\textwidth]{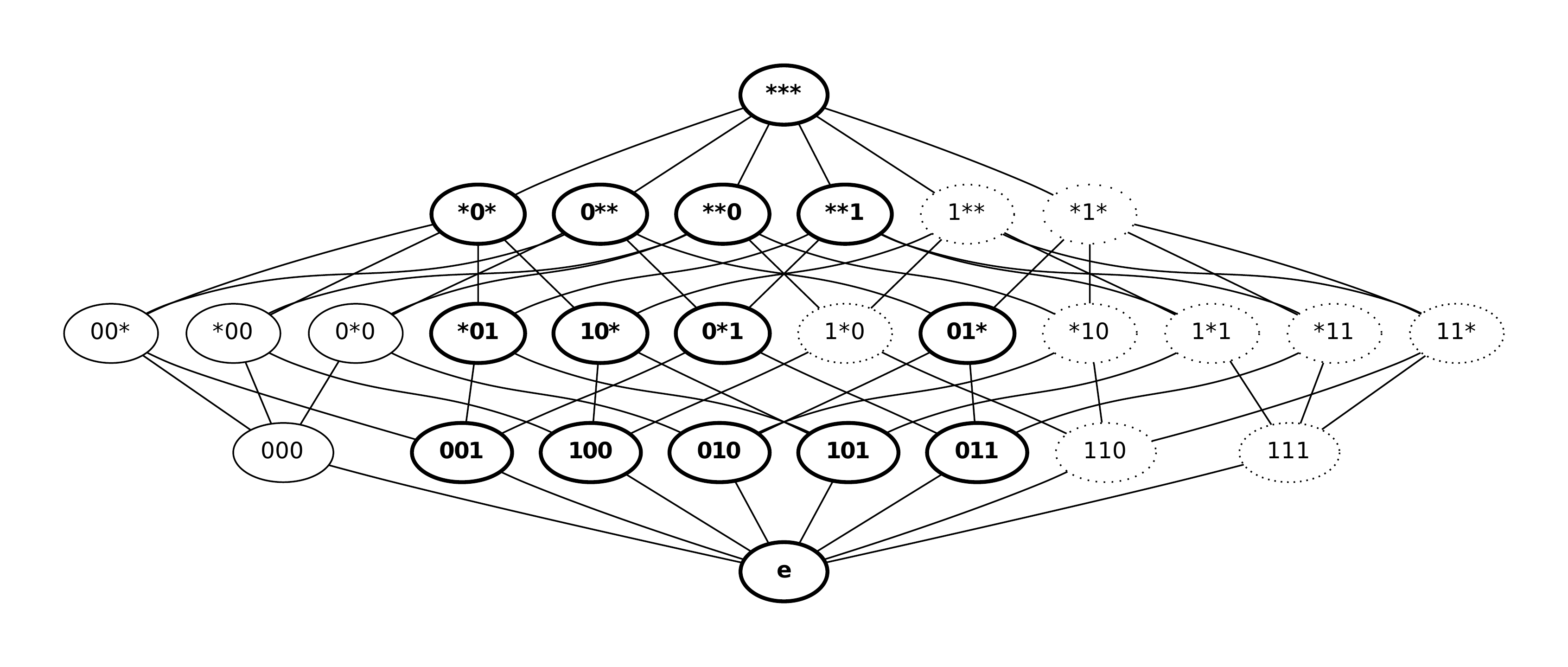}
    \caption{The search space for all schema of length three arranged as a complete lattice. The schemata surrounded by bold circles are those sampled by the generation $G_t = \{010, 011, 001, 101, 100\}$, that is the elements in the complete sub-lattice $\mathcal{S}(G_t)$. The schemata in dotted circles are those which cannot be reached through any combination of schemata in $\mathcal{S}(G_t)$, while the plain circles show the schemata which can be reached through the combination schema in $\mathcal{S}(G_t)$.  This figure was made using the schematax package.    }
    \label{fig:my_label}
\end{figure*}
\\
\\
It is proposed that the power of GAs comes through the combination of a particular type of schemata, namely building blocks. However, as combining schemata in general is not a good exploratory tool, it sheds a serious doubt on how useful combining building blocks is as a search tool. Specifically, if a building block is not `enclosed' within the schematic lattice defined by the initial generation of a GA, combining schemata alone will not discover it. This suggests the disruption and construction of schemata through the imperfect combination of schemata (via crossover) and mutation may play a vital role in allowing a greater exploration of the search space. This contrasts with the traditional view of GAs, where the disruption of schemata via crossover is traditionally seen as a nuisance as they hinder the combination of good schemata \cite{whitley1994genetic,voset1991punctuated}. There has been some work however which suggests that the construction of novel schemata through crossover play a useful exploratory role \cite{spears1991virtues}. 
\\
\\
In the next section, we demonstrate how the concepts of the schematic completion and the schematic lattice can be used experimentally to observe the schema processing performed by a GA. 
\section{Observing building blocks}
In this section, we use the schematic completion to observe the building blocks during the course of a GA. First, however, to study the building block hypothesis using the above methods, we must more precisely define building blocks. In Holland's framing (definition I.1), the phrases `low order' and `low defining length' could refer to an absolute value, such that `low order' schema are schema whose order is less than say $4$. However, for the purpose of this article we take `low' to be relative to the given generation, so that `low order' refers to schemata with below average order, similarly `low defining length' refers to schemata with below average defining length. Building blocks are then redefined as follows:
\begin{definition}
A building block is a schema with below average order, below average defining length and above average fitness \footnote{It should be noted that when computing the average order, fitness and defining length of schemata we do not include the empty schema nor schema with no wildcards (i.e the individuals of the population). }.  
\end{definition}
Using this definition one can find the building blocks present in a given generation by firstly using the schematic completion on the generation to find {\it all} schemata being tested, then by secondly filtering out the building blocks using the definition above. In our framing, the BBH is a statement about the map from a schematic lattice in generation $t$ to the schematic lattice in generation $t+1$. In particular, the BBH states the building blocks in generation $t+1$ should to some degree result from the combination of building blocks in generation $t$ 
\\
\\
However, before we examine how well building blocks are combined using the above updated version of the building block hypothesis and the notion of schematic blending, it is first interesting to examine how the average order and defining length of building blocks change over the course of a GA. If the order and defining length of building blocks increase, it suggests that building blocks are getting larger and more clumped together (as suggested by the BBH). To investigate, we consider the Canonical GA\footnote{The Canonical GA is a binary GA, with roulette wheel selection, single point crossover and mutation.} \cite{goldberg1988genetic} solving the all ones problem (where the fitness of an individual is the number of ones found in the string). We use binary strings of length $64$ and pick the mutation rate to be $0.005$ and the population size as $30$. We run the GA for $120$ generations, and for each generation, we calculate the schematic completion on each to yield all the schemata being tested by that generation, filter out the building blocks using our modified definition of the BBH above and then calculate the average order and defining length of the set of building blocks. We plot the results in Figure 9 averaged over $20$ simulations, where each simulation is started with a random initial population. In addition, to give the reader an indication of what the building blocks may look like for the all ones problem, we display the 3 building blocks from generations, $0$, $40$, $80$ and $120$ from one simulation in figure 9. As can be seen in Figure 9, the order and defining length of the building blocks is indeed increasing during the course of a GA. It seems, at least for the all ones problem, that the defining length and order of the building blocks quickly increase, and then level out around generation 60. 
\\
\\
The results in figure 9, hint at the BBH being implemented GA. However, it is still unclear if and how many building blocks are being explicitly combined by the genetic operators of a GA. Perhaps the results in figure 9 are not best explained by building blocks being combined but instead it is to be expected when a population becomes less random. In particular, when a set of words is more random the average order and defining length of the schemata will be lower because the words `agree' less. As such, the less random the population becomes (as is the case in a GA working on a reasonable fitness function), the more the words `agree' and thus the schemata as result will have higher order and defining length. It is interesting to note that the building blocks which are found are far less neat than those suggested by \cite{mitchell1992royal} in the royal road fitness functions.
\begin{figure*}[t]
\centering
\begin{subfigure}{0.45\textwidth}
\centering
\includegraphics[width=0.99\textwidth]{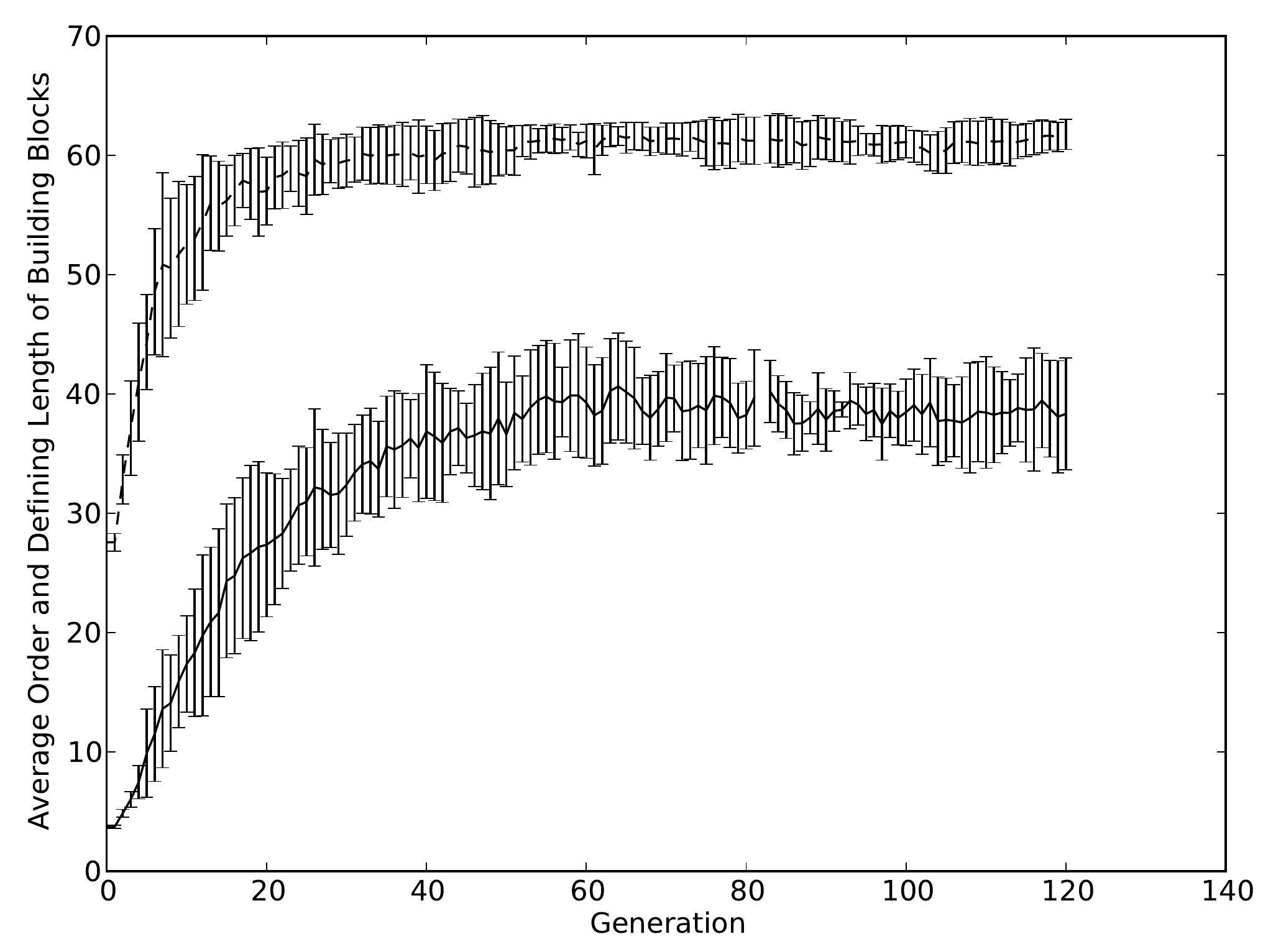}
\caption{The average order and defining length of building blocks during the course of a GA (averaged over $20$ simulations). In this plot, the dashed line represents the average defining length of the building blocks while the solid line represents the average order of the building blocks.}
\end{subfigure}
~
\begin{subfigure}{0.5\textwidth}
\centering

{\footnotesize
{\bf Generation 0:}
***1*****11*****************************************************
***1***************0*******************1************************
********************************0**1****************************
\\
{\bf Generation 40:}
1*101*01011*1***1*******111**1*1*01011***111*1**100*111111*0*1*1
1*101*01011*1***1********110*111*0*0111**111*1**100*111111*0*1*1
1*101*01011*1***1********110*111*0*0*11**111*1**100*111111*0*1*1
\\
{\bf Generation 80:}
\\
11*01*0101**110***0*111*01***1*110**11*10*11*11**00**11****0*1**
11*01*0101**110***0*111**1***1*110**11*10*11*11**00**11****0*1**
11*01**10*101100****111**1**11*1*****1110*11*11***0**11****0*1**
\\
{\bf Generation 120:}
*1*****1*1***1**110011110*11*01*1**1****0******0***1*1111*01*111
*1*****1*11**1**110011110*11*01*1**11***0******00**1*1111*01*111
*1*****1*11**1**110011110*11*01*1**1*1**0******00**1*1111*01*111
}
\caption{A selection of three building blocks from generations: 0, 40, 80 and 120.}

\end{subfigure}
\caption{}
\end{figure*}
\\
\\
For our second experiment, we use the concept of schematic blending to examine exactly how many building blocks from a generation $t+1$ result from the combination of building blocks in generation $t$. We test how well building blocks are combined using different crossover methods. The GA is setup the same as above, however to keep the calculation of the schematic completion and blends tractable, we use a population of size $12$ working over individuals of length $16$. We calculate the building blocks of generation $t$ as before, we then find the set of all schematic blends on the building blocks, let's call this set $B$. It is then checked what percentage of the building blocks in generation $t+1$ are members of $B$. Figure 10 displays the results. The results, in this case, are averaged over 100 simulations. On average only approximately $25\%-35\%$ of building blocks from a generation are created by the combination of building blocks from the previous generation in the case of 1 to 9 point crossover as well as uniform crossover (UX). Interestingly, it is proposed that UX disturbs the combination of building blocks compared to traditional crossover methods \cite{Syswerda:1989,Spears97recombinationparameters}, however we find it combines building blocks equally well as other crossover methods.  In general the reason why building blocks are not combined optimally are several and mostly well known: firstly individuals which are instances of building blocks are not guaranteed to be selected by roulette wheel selection for crossover, thus those building blocks cannot be blended, secondly mutation can disrupt the blending of building blocks if a fixed symbol is mutated after crossover, thirdly crossover is not guaranteed to blend schemata if a suboptimal crossover point is chosen. Thus, we can conclude combining building blocks from generation $t$ only accounts for $25-35\%$ of the building blocks in generation $t+1$, the remaining building blocks from generation $t+1$ are created by other means. 
\\
\\
Probabilistic crossover (PX) which is similar to UX but chooses bits using a weighted probability proportional to the fitness of the parents, combines building blocks the most effectively (due to it picking fitter bits with a higher probability), yet finds the optimal solution in a later generation. It is proposed that ``competent genetic algorithms combine building blocks" \cite{john1992holland}. PX offers a counter-example to this statement as it combines building blocks well, but is not `competent' in that it takes a greater number of generations to find the optimal solution compared to other crossover methods which combine building blocks poorly. It is possible that greater combination of building blocks limits the exploration of the GA. To explain further: the schematic blending lemma tells that combining building blocks will only explore the lower neighbors of preexisting building blocks, thus the GA (with PX) explores this subset well, but does not search other areas of the schematic lattice effectively. Much like Mitchell et al. \cite{mitchell1992royal}  we believe that the over emphasis of building blocks forces the GA into a locally optimal sublattice of the search space. We conclude that the ability of a genetic algorithm to combine building blocks does correspond to how quickly it will find the optimal solution and that the novel creation of schemata  (through methods other than combining building blocks) is vital for a competent GA.  

\begin{figure}[t]
\centering
  \footnotesize{
  \begin{tabular}{l|l|l|}
 
         {\bf Crossover method}& {\bf Solution found} & {\bf \% Building Blocks combined}  \\
         \hline
         1-point & $ 107.3 \pm 25.0 $ & $0.25 \pm 0.16$ \\
         \hline
         2-point & $104.77\pm 28.27$ & $0.37\pm 0.16$ \\
         \hline
         3-point & $101.16\pm 29.13$ & $0.37\pm 0.17 $\\
         \hline
         4-point & $105.06 \pm 29.55$ & $0.37 \pm 0.16$ \\ 
         \hline
         5-point & $94.19\pm 38.36 $& $0.37 \pm 0.16$ \\
         \hline
         6-point &  $104.41 \pm 30.23$& $0.36\pm 0.17$ \\
         \hline 
         7-point &  $103.11 \pm 32.28$ & $0.36 \pm 0.16$ \\
         \hline
         8-point  &$94.37 \pm 36.06$ & $0.36 \pm 0.16 $ \\
         \hline
         9-point & $99.51 \pm 33.19 $ &$0.36 \pm 0.17$ \\  
         \hline
         UX & $ 99.95 \pm 31.20$ & $0.36 \pm 0.16$\\
         \hline
         PX & $119.88 \pm 1.19$ & $0.53 \pm 0.18$ \\
         \hline
         
  \end{tabular}
  }
  \label{tab:1}
   \caption{Time taken for various crossover methods to find the optimal solution and the average percentage of building blocks combined by the respective crossover method. UX here stand for uniform crossover, while PX stands for probabilistic crossover. Each GA is solving the all ones problem on strings of size 16, with a population of size 12 and uses roulette wheel selection. The results are averaged over 100 simulations, each starting with a random initial population. }
\end{figure}

\section{Discussion and Conclusion}
The schematic completion and the schematic lattice, are fruitful in the field of GAs both theoretically, as demonstrated in the mathematics presented in sections 4, and experimentally as demonstrated in section 5. It seems using both methods however, inconsistencies are found the original schema processing theory for GAs. Specifically, section 4 shows that combining schemata (and thus building blocks) explicitly is not a good method to explore the search space of schemata as only the lower neighbors of pre-existing schemata can be reached. While section 5 shows that `competent' GAs do not seem to be very concerned with combining building blocks to begin with, as only approximately 25-35\% of building blocks seem to come from the combination of the previous generations building blocks for most crossover methods. In addition, an increase in the combination of building blocks (as seen in PX) does not correspond with an increase in efficiency of a GA, rather, it hinders the GA. The reason for this follows from the schematic blending lemma, specifically combining building blocks only explores a subset of the search space, thus the more a GA combines building blocks the more it gets stuck in this subset. Finally, the increase of order and defining length of building blocks over time (as is seen in figure 9), is better explained by a decrease in randomness in the population rather than by building blocks being combined. However, we believe the most significant contribution of this article are the methods introduced to explicitly calculate the schemata present in a population and the identification underlying lattice structures involved in schema processing. We hope that these methods will deepen the understanding of GAs.  If the reader is interested in these methods, we encourage them to exploit the schematax software which is introduced in the following appendix section.

\section{Schematax - A Python Software package for schemata}
To complement this article, we introduce an open source python package which implements schemata and all of their properties defined above. Importantly the package allows one to compute the schematic completion and to draw the schematic lattice, it can downloaded from \url{https://github.com/iSTB/python-schemata}.
\\
\\
Naively calculating the schematic completion using the definitions given above requires iterations over the powerset and thus is very computationally expensive ($\mathcal{O}(2^n)$). Thus we introduce an algorithm, algorithm 1 (below), to compute the schematic completion. This algorithm is based on the algorithm presented in \cite{nourine1999fast}. In this algorithm, the join operation is used on each pair of schemata. This algorithm exploits the commutativity ($x \vee y = y \vee x$) and idempotency ($x \vee x = x$) of the join operation. Meaning we do not have to compute $x \vee x$, and if $x \vee y$ is computed, we do not need to compute $y \vee x$. This allows the inner loop only to loop over a subset of schemata. Additionally, we exploit the the atomistic nature of the schematic lattice to build the lattice from the bottom up. 

\begin{proposition}
Algorithm 1's worst case time complexity is $\mathcal{O}(nkN^2)$ where $N$ is the total number of schemata found, $k$ is the length of the strings and $n$ the initial number of strings, that is the size of $P$. 
\end{proposition}
\begin{proof}
The outer loop (line 4) loops over $P$, thus clearly looping $n$ times. For each $n$, the inner loop (line 5) loops $p_c \cdot N-c$ times, where $p_c \leq 1$ is the proportion of N currently found. $join(x,y)$ takes $k$ steps (line 6), where $k$ is the length of the strings. To check if $s$ is not in the current set of schemata found takes $p_c \cdot N$ steps (line 7). So we have: 
$$ \mathcal{O}\left(\sum_{c =1}^{n} (p_c\cdot N-c)\cdot k \cdot (p_c\cdot N) \right)$$
$$ =\mathcal{O}\left(\sum_{c =1}^{n} N \cdot k \cdot  N \right)$$
$$ =\mathcal{O}\left(N^2 \cdot k \sum_{c =1}^{n}1 \right)$$
$$ = \mathcal{O}(nkN^2) $$
\end{proof}
To draw the schematic lattice we exploit the Graphviz software \cite{Gansner00anopen} which allows one to to draw aesthetically pleasing lattices efficiently. In addition, we use the package cython to transfer the Python code into C \cite{behnel2010cython}. This dramatically increases the efficiency of the schematax package. For more information the reader is referred to the documentation of the software package at: \url{https://github.com/iSTB/python-schemata}. 

\begin{algorithm}
\caption{Schematic completion}
\label{Schematic completion}
\begin{algorithmic}[1]
\Procedure{Complete(P)}{}
\State $\textit{schemata} \gets  P$
\State $c \gets  1$
\For {$x \textit{ in P}$}
\State $check \gets \textit{ in schemata[c:]} $
\For {$y \textit{ in check}$}

\State $s \gets join(x,y)$ 

\If {$s \notin \textit{schemata}$}
\State $\textit{schemata} = \textit{schemata} \cup \{s\}$
\EndIf
\EndFor
\State $c \gets  c+1$
\EndFor

\State \Return  $\textit{schemata} \cup \{\epsilon_*\}$

\EndProcedure
\end{algorithmic}
\end{algorithm}

\begin{algorithm}
\caption{Join}
\label{Join}
\begin{algorithmic}[1]
\Procedure{$join(x,y)$}{}
\State $k \gets  length(x)$
\State $s \gets  \textit{empty string} $
\For {$\textit{for i in [0, \dots ,k]}$}

\If {$x[i] = y [i]$}
    \State {$s[i] \gets x[i]$}
\Else 
    \State {$s[i] \gets *$}

\EndIf

\EndFor

\State \Return  $s$

\EndProcedure
\end{algorithmic}
\end{algorithm}

\section*{Acknowledgment}
This work was supported by the Marie Curie Initial Training Network FP7-PEOPLE-2013-ITN (CogNovo, grant number 604764), the Engineering and Physical Sciences Research Council (BABEL, grant number EP/J004561/1), and the University of Plymouth
(through a PhD studentship to Jack McKay Fletcher). We also wish to thank Diego Maranan, Sue Denham and John Matthias for their valuable comments.

\ifCLASSOPTIONcaptionsoff
  \newpage
\fi



%


\FloatBarrier
\bibliographystyle{plain} 
\bibliography{ref}{}

\end{document}